\documentclass{colt2014} 

\usepackage{graphicx}
\usepackage{amsfonts}
\usepackage[mathscr]{euscript}
\usepackage{algorithmic}
\usepackage{algorithm}

\newcommand{\R}{\mathbb{R}}
\newcommand{\Prob}{\mathbb{P}}
\newcommand{\Esp}{\mathbb{E}}

\newcommand{\N}{\mathbb{N}}
\newcommand{\Q}{\mathbb{Q}}
\newcommand{\K}{\mathbb{K}}
\newcommand{\C}{\mathbb{C}}
\newcommand{\vs}[1]{\ensuremath{\mathbf{\boldsymbol{#1}}}}
\renewcommand{\v}[1]{\ensuremath{\mathbf{#1}}}

\newcommand{\e}[1]{\v{\ensuremath{\boldsymbol{\mathscr{#1}}}}}
\newcommand{\comment}[1]{}
\newtheorem*{lemma*}{Lemma}
\newtheorem*{theorem*}{Theorem}
\newtheorem{condition}{Condition}

\DeclareMathOperator*{\argmax}{arg\,max}

\title[Learning Negative Mixture]{Learning Negative Mixture Models by Tensor Decompositions}
\usepackage{times}

 \coltauthor{\Name{Guillaume Rabusseau}
   \Email{guillaume.rabusseau@lif.univ-mrs.fr}\\ \Name{Fran\c{c}ois Denis} \Email{francois.denis@lif.univ-mrs.fr}\\
 \addr Aix Marseille Universit\'e, CNRS, LIF, 13288 Marseille
  Cedex 9, FRANCE
 }

\begin{document}

\maketitle

\begin{abstract}
  This work considers the problem of estimating the parameters of
  \emph{negative mixture models}, i.e. mixture models that possibly
  involve negative weights. The contributions of this paper are as
  follows. (i) We show that every rational probability distributions
  on strings, a representation which occurs naturally in spectral
  learning, can be computed by a negative mixture of at most two
  probabilistic automata (or HMMs).  (ii) We propose a method to
  estimate the parameters of negative mixture models having a specific
  tensor structure in their low order observable moments. Building
  upon a recent paper on tensor decompositions for learning latent variable models, 
  we extend this work to the broader setting
  of tensors having a symmetric decomposition with positive \emph{and
    negative} weights. We introduce a generalization of the
  \emph{tensor power method} for complex valued tensors, and establish
  theoretical convergence guarantees. (iii) We show how our approach
  applies to \emph{negative Gaussian mixture models}, for which we
  provide some experiments.
\end{abstract}


\section{Introduction}

\emph{Mixture models}, such as Gaussian mixture model, are widely used
in statistics and machine learning~\cite{McLachlanPeel00}. Given a parametric family of
probability distributions ${\cal D}$, a mixture is defined by the number of components $k\geq 1$,
probabilities $p_1, \ldots p_k\in [0,1]$ satisfying $p_1+\ldots +p_k=1$
and distributions $F_1, \ldots, F_k$ from ${\cal D}$. 
Given a sample drawn from a target mixture model, the
parameters are usually fit by using the EM algorithm.

Let $f_1, \ldots, f_k$ be the \textsc{pdf} associated with $F_1, \ldots,
F_k$.  It may happen that the function $p_1f_1 + \ldots + p_kf_k$
remains positive even if some of the weights become negative, and
still defines a probability density.  We call \emph{negative
  mixtures} such distributions. Only a few papers investigate negative mixtures~\cite{DBLP:conf/mldm/ZhangZ05,conf/wpnc/MullerADNP12,Jiang1999227,Jevremovic1991}. It can easily be seen that every
negative mixture can be written as a negative mixture
 of 2 positive one.  
We show that negative mixtures naturally occur in spectral learning.

Let $\Sigma$ be a finite alphabet and let $\Sigma^*$ denote the set of
strings built over $\Sigma$. A probability distribution $p$ defined on
$\Sigma^*$ is said to be \emph{rational} if it admits a \emph{linear
  representation}, i.e. if there exists an integer $n\geq 1$, vectors
$\vs{\iota},\vs{\tau}\in \R^n$ and matrices $\v{M}_x\in \R^{n\times n}$ associated with
each letter $x\in \Sigma$ such that $p(u_1\ldots
u_l)=\vs{\iota}^{\top}\v{M}_{u_1}\ldots \v{M}_{u_l}\vs{\tau}$~\cite{DBLP:journals/fuin/DenisE08}. It can easily be shown that any
probability distribution defined by a hidden Markov model (HMM), or equivalently, by a probabilistic
automaton, is rational. However, there exist rational probability
distributions that cannot be computed by a HMM. The spectral learning
algorithms used to infer probability distributions from a sample of
strings generally output rational probability
distributions.
Positive and negative mixtures of rational distributions are
rational. Positive mixtures of distributions computed by HMMs can be
computed by HMMs. In this paper, we show that every rational
distribution $p$ is a negative mixture $(1+w)p_{H_1}-wp_{H_2}$ of two
distributions computed by HMMs. So, negative mixtures occur
naturally. How the parameters of the target model can be fit?

In a recent paper, it has been shown that the parameters of a number
of latent variable models, including Gaussian mixture models and
HMMs, can easily be estimated from tensor decomposition of
low-order moments of the data~\cite{HsuKakade}. Typically, if $\v{x}$ is drawn according to
the (positive) mixture $p_1{\cal N}(\vs{\mu}_1,\sigma) + \ldots + p_k{\cal
  N}(\vs{\mu}_k,\sigma)$ of spherical Gaussians whose centers $\vs{\mu}_i$ are
linearly independent, the tensors
$\v{M}_2=\sum_{i=1}^kp_i\vs{\mu}_i\otimes\vs{\mu}_i$ and
$\e{M}_3=\sum_{i=1}^kp_i\vs{\mu}_i\otimes\vs{\mu}_i\otimes\vs{\mu}_i$ can be expressed as
functions of the moments $\Esp[\v{x}\otimes \v{x}]$ and $\Esp[\v{x}\otimes
\v{x}\otimes \v{x}]$~\cite{HsuKakade2}. Using the fact that $\v{M}_2$ is positive semidefinite, $\e{M}_3$ can be reduce to a tensor $\e{\widetilde{M}}_3$ admitting an
\emph{orthonormal} decomposition
$\e{\widetilde{M}}_3=\sum_{i=1}^k\widetilde{p}_i\widetilde{\vs{\mu}}_i\otimes\widetilde{\vs{\mu}}_i\otimes\widetilde{\vs{\mu}}_i$,
where $\widetilde{\vs{\mu}}_i^{\top}\widetilde{\vs{\mu}}_j=\delta_{ij}$, and from which the
original parameters $p_i$ and $\vs{\mu}_i$ can be recovered. Lastly,
it is shown that a decomposition of an orthogonally decomposable tensor can
quickly and robustly be approximated by means of a \emph{tensor power
  method}. These results induce a learning scheme, which appears as a
generalization of the spectral learning approach: from a sample $S$,
compute estimates of  $\v{M}_2$ and $\e{M}_3$, compute  $\e{\widetilde{M}}_3$ and
use the tensor power method to compute an orthogonal decomposition of
$\e{\widetilde{M}}_3$ from which the parameters of the target can be
estimated. 

Each step of the previous scheme strongly use the facts that the
weights $p_i$ are positive and that the $\vs{\mu}_i$ are linearly
independent. We extend it to the case where the weights $p_i$ may be
negative.  The extension is not straightforward since it needs to use
complex square roots of negative real numbers and to introduce non-hermitian
quadratic forms. 

Given the tensors
$\v{M}_2=\sum_{i=1}^kp_i\vs{\mu}_i\otimes\vs{\mu}_i$ and
$\e{M}_3=\sum_{i=1}^kp_i\vs{\mu}_i\otimes\vs{\mu}_i\otimes\vs{\mu}_i$ where the vectors
$\vs{\mu}_i\in \R^n$ are still linearly independent but where the weights
$p_i$ may be negative, we first show how $\e{M}_3$ can be reduced to a
complex-valued pseudo-orthonormal decomposable tensor, i.e. of
the form
$\sum_{i=1}^k\widetilde{p}_i\widetilde{\vs{\mu}}_i\otimes\widetilde{\vs{\mu}}_i\otimes\widetilde{\vs{\mu}}_i$,
where the vectors $\widetilde{\vs{\mu}}_i\in \C^k$ satisfy $\widetilde{\vs{\mu}}_i^{\top}\widetilde{\vs{\mu}}_j=\delta_{ij}$ (for any
vector $\vs{\mu}\in \C^k, \vs{\mu}^{\top}\vs{\mu}\in \C$ since $\vs{\mu}^{\top}$
\textbf{is not} the conjugate transpose of $\vs{\mu}$) and where the
weights $\widetilde{p}_i$ are non-zero complex numbers. Then, we show
how the tensor power method can be adapted to the complex case, with
equivalent convergence guarantees. We deduce from these results a
learning scheme for negative mixtures.  To illustrate this analysis, we experiment our decomposition algorithm
on negative mixtures of spherical Gaussian models and we show how
estimates of a negative mixture target can be inferred from data.  

The paper is organized as follows: preliminaries on rational
probability distributions and tensor decomposition learning methods
are given in Section~\ref{sec:prel}; negative mixtures are introduced
in Section~\ref{seg:negmix} and two introductive examples are
developed; the adaptation of the tensor decomposition learning scheme
to negative mixtures and the main results of the paper are given in
Section~\ref{sec:mainresults}; an application to negative mixtures of
spherical gaussians and some experiments are provided in
Sections~\ref{sec:learning}~and~\ref{sec:expe}; a conclusion ends the paper. 


\section{Preliminaries}\label{sec:prel}




\subsection{Rational probability distributions on strings}

Let $\Sigma$ be a finite alphabet and $\Sigma^*$ denote the set of
all finite strings built over $\Sigma$.  A \emph{series} is a mapping
$r:\Sigma^* \to \R$. A non negative series
$r$ is \emph{convergent} if the sum $\sum_{w\in\Sigma^*}r(w)$ is bounded; its
limit is denoted by $r(\Sigma^*)$. A \emph{probability distribution}
over $\Sigma^*$ is a non-negative series that converges to 1.
A series $r$ over $\Sigma$ is \emph{rational} if there exists an
integer $n\geq 1$, two vectors $\vs{\iota},\vs{\tau}\in \R^n$ and a matrix $\v{M}_x\in
\R^{n\times n}$ for each $x\in \Sigma$ such that for all $u=u_1\ldots u_n\in
\Sigma^*$, $r(u)=\vs{\iota}^T\v{M}_{u_1}\ldots \v{M}_{u_n} \vs{\tau}$~\cite{opac-b1086956}. The
triplet $\langle \vs{\iota},(\v{M}_x)_{x\in \Sigma},\vs{\tau}\rangle$ is called an $n$-dimensional
\emph{linear representation} of $r$. An $n$-states probabilistic
automaton (PA) can be defined as an $n$-dimensional
\emph{linear representation} $\langle \vs{\iota},(\v{M}_x)_{x\in
  \Sigma},\vs{\tau}\rangle$ whose coefficients are all non-negative
and satisfy the following syntactical conditions $$\vs{\iota}^{\top}{\mathbf 1}=1,
\v{I}-\v{M}_{\Sigma}\textrm{ is invertible and }(\v{I}-\v{M}_{\Sigma})^{-1}\vs{\tau}={\mathbf 1}
$$ where ${\mathbf 1}=(1, \ldots 1)^{\top}\in \R^n$ and
$\v{M}_{\Sigma}=\sum_{x\in \Sigma}\v{M}_x$. Hidden Markov Models (HMM) and PAs define the same probability distributions~\cite{DBLP:journals/pr/DupontDE05}. There exist rational probabilistic distributions
that cannot be computed by a PA or a HMM (see Appendix~\ref{supmat:rat}). 


\subsection{Moments method and tensor decomposition}

See \cite{Kolda09} for references on tensor decomposition. Let us denote by $\bigotimes^p \K^{n}$  the \emph{p-th order tensor product} of the vector space $\K^n$, where $\K=\R$ or $\C$. A tensor $\e{T}\in \bigotimes^p \K^n$ can be described by a $p$-way array of scalars $t_{i_1,\cdots,i_p}\in\K$ for $i_1, \cdots, i_p \in [n]$, where $[n]$ denotes the set of integers between $1$ and $n$. A tensor is \emph{symmetric} if its multi-way array representation is invariant under permutation of the indices. Given $\v{v}^{(1)}, \ldots, \v{v}^{(p)} \in \K^n$, the tensor $\v{v}^{(1)}\otimes \cdots \otimes \v{v}^{(p)} \in \bigotimes^p\K^n$ is defined by the $p$-way array $(v_{i_1}^{(1)} v_{i_2}^{(2)}\ldots v_{i_p}^{(p)})_{i_1, \cdots, i_p \in [n]}$. For a vector $\v{v}\in\K^n$, let $\v{v}^{\otimes ^p} = \v{v}\otimes \cdots \otimes \v{v}$ denote the $p$-th tensor power of $\v{v}$. In particular, $\v{v}\otimes \v{v} $ can be identified with the matrix $\v{v}\v{v}^\top$. Let $\v{x}$ be a $\R^n$-valued random variable, its moment of order $m$ is defined as the tensor $\Esp[\v{x}^{\otimes m}] \in \bigotimes^m\R^n$.

For any integers $m_1,\cdots,m_p\geq 1$, every $p$-th order tensor  $\e{T}\in \bigotimes^p \K^n$ induces a multilinear map  $\e{T}: \K^{n\times m_1} \times \cdots \times \K^{n\times m_p} \to \K^{m_1 \times \cdots \times m_p}$ defined by $\e{T}(\v{A}^{(1)},\cdots,\v{A}^{(p)})_{i_1,\cdots,i_p} = \sum_{j_1,\cdots,j_p\in [n]} t_{j_1,\cdots,j_p}a^{(1)}_{j_1 i_1}\cdots a^{(p)}_{j_p i_p}$ where each $i_k \in [m_k]$ for $k\in [p]$. In particular, $$\textrm{if } \e{T} = \sum_{i=1}^k \lambda_i \v{v}^{(1)}_i\otimes \cdots \otimes \v{v}_i^{(p)}\textrm{ then }\e{T}(\v{A}_1,\cdots,\v{A}_p) = \sum_{i=1}^k \lambda_i (\v{A}_1^\top \v{v}^{(1)}_i)\otimes\cdots\otimes(\v{A}_p^\top \v{v}^{(p)}_i).$$

The \emph{rank} of a tensor $\e{T}\in \bigotimes^p \K^n$ is the smallest integer $k$ such that $\e{T}$ can be written as $\e{T} = \sum_{i=1}^k \lambda_i \v{v}^{(1)}_i\otimes \cdots \otimes \v{v}_i^{(p)}$ with $\lambda_i\in\K$ and $\v{v}^{(1)}_i,\cdots,\v{v}_i^{(p)} \in \K^n$. The \emph{symmetric rank} of a symmetric tensor $\e{T}$ is the smallest integer $k$ such that $\e{T}$ can be written as $\e{T} = \sum_{i=1}^k \lambda_i \v{v}_i ^{\otimes ^p}$ with $\lambda_i\in\K$ and $\v{v}_i\in \K^n$. It has been shown that computing the rank of a tensor is NP-hard and it is conjectured that computing the symmetric rank is also NP-hard~\cite{Hillar:2013}. However, if a real-valued third-order tensor $\e{T}$ has a \emph{symmetric orthonormal decomposition}, i.e. $\e{T} = \sum_{i=1}^k \lambda_i \v{v}_i^{\otimes 3}$ with $\lambda_i\in\R$, $\v{v}_i \in \R^n$ and $\v{v}_i^\top \v{v}_j = \delta_{ij}$ for all $i,j\in [k]$, it has been shown in \cite{HsuKakade} that this decomposition can be recovered by several methods,  both efficient and robust to noise, such as the \emph{tensor power method} (see Section~\ref{sec:power} below). Moreover, they show that any \emph{symmetric independent decomposition} $\e{T} = \sum_{i=1}^k \lambda_i \v{v}_i^{\otimes 3}$ (where the $\v{v}_i$'s are independent but not necessarily orthonormal) can be recovered if we have access to the second order tensor $\v{M} = \sum_{i=1}^k \lambda_i \v{v}_i \v{v}_i^\top$.

\begin{theorem}\label{th:1}\cite{HsuKakade}
Let $\v{v}_1, \ldots, \v{v}_k$ be linearly independent vectors of $\R^n$, $\lambda_1, \ldots, \lambda_k$ be positive scalars, $\v{M}_2=\sum_{i=1}^k\lambda_i \v{v}_i\otimes \v{v}_i$ and $\e{M}_3=\sum_{i=1}^k\lambda_i \v{v}_i^{\otimes 3}$, let $\v{W}\in \R^{n\times k}$ be a matrix such that $\v{M}_2(\v{W},\v{W})=\v{I}_k$, the $k\times k$ identity matrix, and let $\vs{\nu}_i=\sqrt{\lambda_i}\v{W}^{\top}\v{v}_i$ for $i\in [k]$. Then, $\e{M}_3(\v{W},\v{W},\v{W})=\sum_{i=1}^k\lambda_i^{-1/2}\vs{\nu}_i^{\otimes 3}$ is an orthonormal decomposition from which the parameters $\lambda_i$ and $\v{v_i}$ can be computed. 
\end{theorem}

\subsection{Learning mixtures of spherical Gaussians}
\label{GaussMixSection}

The \emph{spherical Gaussian mixture model} is specified as follows:  let $k\geq 1$ be the number of components, and for $i\in [n]$, let $p_i>0$ be the probability of choosing the component ${\cal N}(\vs{\mu}_i,\sigma^2_i \v{I})$ where $\vs{\mu_i}\in \R^n$, $\sigma^2_i > 0$ and $\v{I}\in\R^{n\times n}$ is the identity matrix. 

Assuming that the component mean vectors $\vs{\mu}_i$ are linearly independent, the following result is proved in \cite{HsuKakade2}.

\begin{theorem} \label{GaussMixThm}
The average variance $\bar{\sigma}^2 = \sum_{i=1}^k p_i \sigma^2_i$ is the smallest eigenvalue of the covariance matrix $\Esp [(\v{x} - \Esp [\v{x}]) (\v{x} - \Esp [\v{x}]) ^\top]$. Let $\v{v}$ be any unit-norm eigenvector corresponding to $\bar{\sigma}^2$ and let \begin{align*}
\v{m}_1 &= \Esp[\v{x}(\v{v}^\top (\v{x} - \Esp [\v{x}]))^2], \ \ \ \ \  
\v{M}_2 = \Esp[\v{x}\otimes\v{x}] - \bar{\sigma}^2 \v{I}, \ \mbox{ and }\\
\e{M}_3 &=  \Esp[\v{x}\otimes\v{x}\otimes\v{x}] - \sum_{i=1}^n [\v{m}_1\otimes \v{e}_i \otimes \v{e}_i + \v{e}_i \otimes \v{m}_1\otimes \v{e}_i +  \v{e}_i \otimes \v{e}_i \otimes \v{m}_1]
\end{align*}
where $\v{e}_1,\cdots,\v{e}_n$ is the coordinate basis of $\R^n$. Then, $$\v{m}_1 = \sum_{i=1}^k p_i \sigma^2_i \vs{\mu}_i,\ \ \ \ \v{M}_2 = \sum_{i=1}^k p_i \vs{\mu}_i \otimes \vs{\mu}_i\  \textrm{, and }\ \ \ \e{M}_3 = \sum_{i=1}^k p_i \vs{\mu}_i \otimes \vs{\mu}_i \otimes \vs{\mu}_i.$$
\end{theorem}   

The previous results induce a learning scheme: (i) estimate $\v{m}_1$, $\v{M}_2 $ and $\e{M}_3$ from the learning data; (ii) compute an orthonormal decomposition as in Theorem~\ref{th:1}; (iii) use the tensor power method to compute the mean vectors $\vs{\mu}_i$ and the probabilities $p_i$  and (iv) use $\v{m}_1$ to recover the variance parameters $\sigma^2_i$.




\section{Negative mixtures}\label{seg:negmix}
Given a finite set of probability density functions $f_1 \ldots, f_k$,
and non negative weights $w_1, \ldots, w_k$ satisfying $w_1 + \ldots +
w_k=1$, $w_1f_1+ \ldots + w_kf_k$ is a probability density function
called a \emph{finite mixture}.  It may happen that $w_1f_1+ \ldots +
w_kf_k$ defines a \textsc{pdf} even if some weights
are negative. We call such a function a \emph{negative} or a
\emph{generalized} mixture.

For example, if $f$ and $g$ are two \textsc{pdf}
satisfying $g\leq cf$ for some $c>1$, then $\alpha f-(\alpha - 1) g$ is a
negative mixture for any $0\leq \alpha - 1\leq (c-1)^{-1}.$ It can
easily be shown, by grouping the positive and negative weights respectively, that any negative mixture can be written as a negative
mixture of \emph{two} positive mixtures:$$\sum_{i=1}^k\alpha_if_i-
\sum_{j=1}^h\beta_jg_j=A\left(\sum_{i=1}^k\frac{\alpha_i}{A}f_i\right)-B\left(\sum_{j=1}^h\frac{\beta_j}{B}g_j\right)$$
where $\alpha_i, \beta_j>0$, $A= \sum_{i=1}^k\alpha_i$,
$B=\sum_{j=1}^h\beta_j$ and $A-B=1$.

If $f,g$ and $\alpha$ are known, and if we have access to a random
generator ${\cal D}_f$, then Algorithm~\ref{algo:sim} simulates the
distribution ${\cal
  D}_{\alpha f-(\alpha-1) g}$ by rejection sampling. 

\begin{algorithm}
\caption{Simulating a negative mixture}
\label{algo:sim}
\begin{algorithmic}[H]
\STATE drawn $\leftarrow$ false\\
\WHILE{not drawn}
\STATE draw $x$ according to ${\cal D}_f$
\STATE draw $e$ uniformly in $[0,1]$
\IF{$e\alpha f(x)\geq (\alpha-1) g(x)$}
\STATE drawn $\leftarrow$ true
\ENDIF
\ENDWHILE
\RETURN $x$
\end{algorithmic}
\end{algorithm}


\subsection{Negative mixtures and rational distributions on strings}\label{sec:rat}
We show below that every rational probability distribution
on strings can be generated by the generalized mixture of at most two
probabilistic automata. The proof relies on the following lemmas. 
\begin{lemma}\label{lem:difplus}
Any rational series is the difference of two rational series with non
negative coefficients. 
\end{lemma}

\begin{proof} 
For any real number $x$, let $x^+=\max\{x,0\}$ and $x^-=\max\{-x,0\}$. So,
$x=x^+-x^-$. These operators are extended to vectors and matrices by
applying them to all their coefficients. 

Let $\langle \vs{\iota},(\v{M}_x)_{x\in \Sigma},\vs{\tau}\rangle$ be an
$n$-dimensional representation of a rational series $r$. Let us
define 
$$\widetilde{\vs{\iota}}_1=\left(
  \begin{array}{c}
    \vs{\iota}^+\\\vs{\iota}^-
  \end{array}
\right), \widetilde{\vs{\iota}}_2=\left(
  \begin{array}{c}
    \vs{\iota}^-\\\vs{\iota}^+
  \end{array}
\right), \widetilde{\vs{\tau}}=\left(
  \begin{array}{c}
    \vs{\tau}^+\\\vs{\tau}^-
  \end{array}
\right)\textrm{ and } \widetilde{\v{M}}_x=\left(
  \begin{array}{cc}
    \v{M}_x^+&\v{M}_x^-\\\v{M}_x^-&\v{M}_x^+
  \end{array}
\right)\textrm{ for each }x\in \Sigma.$$

\normalsize
 Let $r^+$ (resp. $r^-$) be the rational series
defined by the linear representation $\langle
\widetilde{\vs{\iota}_1},(\widetilde{\v{M}}_x)_{x\in \Sigma},\tilde{\vs{\tau}}\rangle$ (resp. $\langle
\widetilde{\vs{\iota}_2},(\widetilde{\v{M}}_x)_{x\in \Sigma},\widetilde{\vs{\tau}}\rangle$). Then,
$r=r^+-r^-$. Indeed, it can easily be checked that for any vectors
$\v{u}_1, \v{u}_2, \v{v}_1, \v{v}_2\in \R^n$, $$\widetilde{\v{M}}_x \left(
  \begin{array}{c} \v{u}_1\\\v{u}_2 \end{array} \right)=\left(
  \begin{array}{c} \v{v}_1\\\v{v}_2 \end{array} \right)\Rightarrow
\v{M}_x(\v{u}_1-\v{u}_2)=\v{v}_1-\v{v}_2.$$Therefore, for any $w=w_1\ldots w_n\in
\Sigma^*$, $$\widetilde{\v{M}}_{w_1} \ldots \widetilde{\v{M}}_{w_n}\widetilde{\vs{\tau}}=\left(
  \begin{array}{c} \v{v}_1\\\v{v}_2 \end{array} \right)\Rightarrow
\v{M}_{w_1}\ldots \v{M}_{w_n} \vs{\tau}=\v{v}_1-\v{v}_2.$$Since $$(\widetilde{\vs{\iota}_1}^{\top}-\widetilde{\vs{\iota}_2}^{\top}) \left(
  \begin{array}{c} \v{v}_1\\\v{v}_2 \end{array} \right)=\vs{\iota}^{\top}(\v{v}_1-\v{v}_2),$$it can easily be checked that $$r^+(w)-r^-(w)=(\widetilde{\vs{\iota}_1}^{\top}-\widetilde{\vs{\iota}_2}^{\top}) \widetilde{\v{M}}_{w_1} \ldots \widetilde{\v{M}}_{w_n}\widetilde{\vs{\tau}}=r(w).$$
\end{proof}
However, even if the non negative series $r$ is convergent, the series
$r^+$ and $r^-$ obtained from the previous construction can be
divergent (see an example in Appendix~\ref{supmat:rat}). It has
been shown in~\cite{DBLP:journals/iandc/BaillyD11} that if a rational series $r$ is absolutely
convergent, then it can always be computed by a linear representation
$\langle \vs{\iota},(\v{M}_x)_{x\in \Sigma},\vs{\tau}\rangle$ such that $\langle
|\vs{\iota}|,(|\v{M}_x|)_{x\in \Sigma},|\vs{\tau}|\rangle$ defines a positive convergent
series $s$. In that case, $r^+$ and $r^-$ are bounded by $s$ and are
convergent. Let $s^+=r^+(\Sigma^*)$, $s^-=r^-(\Sigma^*)$ and let
$p^+=r^+/s^+$ and $p^-=r^-/s^-$ : $p^+$ and $p^-$ are rational
probability distributions and if $r$ is itself a probability distribution, we
have $s^+-s^-=1$ and $r$ is equal to the generalized mixture
$s^+p^+-s^-p^-$.  It remains to prove that $p^+$ and $p^-$ can be
computed by a probabilistic automaton. 
\begin{lemma}Let $\langle\vs{\iota},(\v{M}_x)_{x\in \Sigma},\vs{\tau}\rangle$ be
  an $n$-dimensional
  minimal \emph{non negative} linear representation of a probability
  distribution $p$. Let $\vs{\lambda}=(\v{I}-\v{M}_{\Sigma})^{-1}\vs{\tau}$ and
  $\v{D}=diag(\vs{\lambda})$. 
  
  Then, $\langle \v{D}\vs{\iota},(\v{D}^{-1}\v{M}_x\v{D})_{x\in
    \Sigma},\v{D}^{-1}\vs{\tau}\rangle$ is a probabilistic automaton that
  recognizes $p$.
\end{lemma}
\begin{proof}
The minimality of the representation implies that $\v{D}$ is
invertible. It is clear that the new representation recognizes $p$
since $(\v{D}\vs{\iota})^{\top}\v{D}^{-1}\v{M}_{x_1}\v{D}\ldots \v{D}^{-1}\v{M}_{x_n}\v{D}
\v{D}^{-1}\vs{\tau}=\vs{\iota}^{\top}\v{M}_{x_1}\ldots \v{M}_{x_n}\vs{\tau}.$ We
have $(\v{D}\vs{\iota})^{\top}\v{1}=\vs{\iota}^{\top}\vs{\lambda}=1.$ Moreover,
$\v{I}-\v{D}^{-1}\v{M}_{\Sigma}\v{D}=\v{D}^{-1}(\v{I}-\v{M}_{\Sigma})\v{D}$ is invertible and $(\v{I}-\v{D}^{-1}\v{M}_{\Sigma}\v{D})^{-1}\v{D}^{-1}\vs{\tau}=\v{D}^{-1}\vs{\lambda}=\v{1}$.
\end{proof}

Combining the previous lemmas, we obtain the following theorem.
\begin{theorem}
Every rational probability distribution on strings can be generated by the generalized mixture of at most two probabilistic automata. 
\end{theorem}

\comment{
\begin{proof}
Indeed, every rational probability distribution $p$ can expressed as
$p=s^+p^+-(s^+-1)p^-$ where $p^+$ and $p^-$ can be
computed by a probabilistic automaton. 
\end{proof}
}

\subsection{Negative mixtures and gaussians}

Let 
$f$ and $g$ be the \textsc{pdf} of the two $k$-dimensional Gaussian
distributions ${\cal N}(\vs{\mu}_f,\vs{\Sigma}_f)$ and ${\cal N}(\vs{\mu}_g,\vs{\Sigma}_g)$.

For any real number $\alpha>0$,
\begin{equation}
\alpha f(\v{x})- (\alpha-1) g(\v{x}) \geq 0 \label{eq:g}
\end{equation}
if and only if 
$$\exp\left\{ -\frac{1}{2}(\v{x}-\vs{\mu}_f)^\top\vs{\Sigma}_f^{-1}(\v{x}-\vs{\mu}_f)+\frac{1}{2}(\v{x}-\vs{\mu}_g)^\top\vs{\Sigma}_g^{-1}(\v{x}-\vs{\mu}_g)\right\}
\geq \sqrt{\frac{|\vs{\Sigma}_f|}{|\vs{\Sigma}_g|}}[1-1/\alpha].$$
There exists $\alpha > 1$ such that~(\ref{eq:g}) holds for any $\v{x}\in
\R^k$ if and only if
\begin{equation}
-(\v{x}-\vs{\mu}_f)^\top\vs{\Sigma}_f^{-1}(\v{x}-\vs{\mu}_f)+(\v{x}-\vs{\mu}_g)^\top\vs{\Sigma}_g^{-1}(\v{x}-\vs{\mu}_g)\label{eq:2}
\end{equation}
has a finite lower bound which holds if and only if
$(\vs{\Sigma}_g^{-1}-\vs{\Sigma}_f^{-1})$ is positive semi-definite. 

In that case, the minimum $m$ of~(\ref{eq:2}) is attained
for $$\vs{\mu}_0=\vs{\Sigma}_0 (\vs{\Sigma}_g^{-1}\vs{\mu}_g-\vs{\Sigma}_f^{-1}\vs{\mu}_f)$$
where $\vs{\Sigma}_0=(\vs{\Sigma}_g^{-1}-\vs{\Sigma}_f^{-1})^{-1}$,
 and there exists a constant $\lambda$ such that $\lambda g/f$
defines a Gaussian distribution of parameters $\vs{\mu}_0$ and $\vs{\Sigma}_0$. It can be checked that $$m=-(\vs{\mu}_f-\vs{\mu}_g)^\top\vs{\Sigma}_f^{-1}\vs{\Sigma}_0\vs{\Sigma}_g^{-1}(\vs{\mu}_f-\vs{\mu}_g).$$


Note that if the two distributions are distinct, then
$\left(\frac{|\vs{\Sigma}_g|}{|\vs{\Sigma}_f|}\right)^{1/2}e^{m/2}-1<0$. Otherwise, any positive $\alpha$ would be suitable and by
dividing~(\ref{eq:g}) by $\alpha$,  the density of the first
distribution would be everywhere larger than the density of the second,
which cannot happen. Hence every $$\alpha\in
\left]1,\left(1-\sqrt{\frac{|\vs{\Sigma}_g|}{|\vs{\Sigma}_f|}}e^{m/2}\right)^{-1}\right]$$ defines a valid
negative mixture of the two distributions. 

If the gaussians are spherical, i.e. $\vs{\Sigma}_f=\sigma^2_f\v{I}$ and
$\vs{\Sigma}_g=\sigma^2_g\v{I}$, we obtain the following result. 
\begin{proposition}
  $\alpha f(\v{x})- (\alpha-1) g(\v{x}) $ defines a negative mixture iff
  $$\sigma_f>\sigma_g\textrm{ and }1<\alpha\leq \left(1-\frac{\sigma^k_g}{\sigma^k_f} \exp\left\{-\frac{1}{2}\frac{||\vs{\mu}_f-\vs{\mu}_g||^2}{\sigma_f^2-\sigma_g^2}\right\}\right)^{-1}.$$ 
\end{proposition}

\textbf{Example} Let $k=2$, $\vs{\mu}_f=(11.4\ \ -3.4)^{\top}$, $\sigma_f^2=8$,
$\vs{\mu}_g=(11.9\ \ -1.9)^{\top}$, $\sigma_g^2=4$: $\alpha f(\v{x})- (\alpha-1) g(\v{x})
$ defines a negative mixture for any $1<\alpha\leq 1.57$. See figure~\ref{fig_pdf_cvg}.


\section{Negative Mixtures and the Power Method}\label{sec:mainresults}

 We consider systems of the form
\begin{equation}
\label{system}
\v{M}_2 = \sum_{i=1}^k w_i\vs{\mu}_i\otimes\vs{\mu}_i\ \ \  \mbox{ and }\ \ \  \e{M}_3 = \sum_{i=1}^k w_i\vs{\mu}_i\otimes\vs{\mu}_i \otimes\vs{\mu}_i
\end{equation}
where the vectors $\vs{\mu}_1,\cdots,\vs{\mu}_k \in \R^d$ are linearly independent and $w_1,\cdots,w_k\in\R$ are non zero. 

In this section, we show how the parameters $w_i$ and $\vs{\mu}_i$ can be recovered from $\v{M}_2$ and $\e{M}_3$  using a \emph{power method for complex-valued tensors}. 

\subsection{Pseudo-Orthonormalization} 
\label{pseudo-orth}
A set $\{\vs{\nu}_1, \ldots, \vs{\nu}_k\}\subset \C^d$ is \emph{pseudo-orthonormal} iff $\vs{\nu}_i^\top\vs{\nu}_j = \delta_{ij}$ for all $i,j\in [k]$. Note that for any $\vs{\nu}=(\nu_1, \ldots, \nu_d)\in \C^d$, $\vs{\nu}^{\top}\vs{\nu}=\nu_1^2+ \ldots + \nu_d^2\in \C$ and in particular, $\vs{\nu}^{\top}\vs{\nu}\neq ||\vs{\nu}||_2^2 =|\nu_1|^2+ \ldots + |\nu_n|^2$. It can easily be checked that a pseudo-orthonormal set is linearly independent. A tensor decomposition $\e{T} = \sum_{i=1}^k z_i \vs{\nu}_i^{\otimes p}$ of a complex-valued tensor $\e{T} \in \bigotimes^p\C^n$ is \emph{pseudo-orthonormal} if $\{\vs{\nu}_1, \ldots, \vs{\nu}_k\}$ is a pseudo-orthonormal set.

As in \cite{HsuKakade}, we build a whitening matrix $\v{W}$ from $\v{M}_2$, and we use $\v{W}$ to obtain a pseudo-orthonormal decomposition of the tensor $\e{M}_3$. 

 Identifying $\v{M}_2$ with the symmetric rank-$k$ matrix $\sum_{i=1}^k w_i\vs{\mu}_i\vs{\mu}_i^\top$, let $\v{UDU}^\top$ be the eigendecomposition of $\v{M}_2$, where $\v{D}$ is the $k\times k$ diagonal matrix whose diagonal elements are composed of the $k$ non-zero eigenvalues of $\v{M}_2$ and where $\v{U}$ is a $d\times k$ matrix satisfying $\v{U}^\top\v{U}=\v{I}_k$ and $\v{U}\v{U}^{\top}\vs{\mu}_i=\vs{\mu}_i$ for any $i\in [k]$. Let $\v{W} = \v{UD}^{-\frac{1}{2}}\in \C^{d\times k}$ and $\widetilde{\vs{\mu}}_i = w_i^\frac{1}{2}\v{W}^\top\vs{\mu}_i\in \C^k$ for $i\in [k]$ where we consider complex square roots of the negative components of $\v{D}$ and $w_i$: $x^{1/2}=i|x|^{1/2}$ and $x^{-1/2} = (x^{1/2})^{-1} = -i|x|^{-1/2}$ if $x<0$. We have
 \begin{equation*}
 \sum_{i=1}^k \widetilde{\vs{\mu}}_i \widetilde{\vs{\mu}}_i^\top = \v{W}^\top\left(\sum_{i=1}^k w_i\vs{\mu}_i\vs{\mu}_i^\top\right)\v{W}=\v{W}^\top \v{M}_2 \v{W} = \v{I}_k
 \end{equation*}
 hence $\widetilde{\vs{\mu}}_i^\top \widetilde{\vs{\mu}}_j = \delta_{ij}$ for all $i,j\in [k]$. Now let $\e{\widetilde{M}}_3 = \e{M}_3(\v{W},\v{W},\v{W})=\sum_{i=1}^k w_i (\v{W}^\top \vs{\mu}_i)^{\otimes 3} = \sum_{i=1}^k w_i^{-\frac{1}{2}} \widetilde{\vs{\mu}}_i^{\otimes 3}$ which is a pseudo-orthonormal decomposition. 
 
\subsection{Power Method for Complex-Valued Tensors}\label{sec:power}

The following theorem extends Lemma 5.1 of \cite{HsuKakade} to third-order complex-valued tensors having a pseudo-orthonormal decomposition $\e{T} = \sum_{i=1}^k z_i \vs{\nu}_i^{\otimes 3}$. Note that the parameters of such a decomposition are not fully identifiable since $z\vs{\nu}^{\otimes 3} = (-z)(-\vs{\nu})^{\otimes 3}$. 

\comment{
  Throughout this section, we assume the following condition.

 \begin{condition}
 \label{cond}
 The vectors $\vs{\nu}_1, \cdots, \vs{\nu}_k\in\C^n$ are linearly independent and the scalars $z_1,\cdots,z_k\in\C$ are non-zero.
 \end{condition}
}
\begin{theorem}
\label{mainThm}
Let $\e{T} \in \bigotimes^3\C^n$ have a pseudo-orthonormal decomposition $\e{T} = \sum_{i=1}^k z_i \vs{\nu}_i^{\otimes 3}$, and let $T$ be the mapping defined by $T(\vs{\theta}) = \e{T}(I,\vs{\theta},\vs{\theta})$ for any $\vs{\theta} \in \C^n$ . Let $\vs{\theta}_0 \in \C^n$, suppose that $|z_1.\vs{\nu}_1^\top \vs{\theta}_0| > |z_2.\vs{\nu}_2^\top \vs{\theta}_0| \geq \cdots \geq |z_k.\vs{\nu}_k^\top \vs{\theta}_0| > 0$. For $t=1,2,\cdots$, define
\begin{equation}
\label{iterator}
\vs{\theta}_t = \frac{T(\vs{\theta}_{t-1})}{ [T(\vs{\theta}_{t-1})^\top T(\vs{\theta}_{t-1})]^{\frac{1}{2}}}\ \  \text{ and }\ \ \lambda_t = \e{T}(\vs{\theta}_t,\vs{\theta}_t,\vs{\theta}_t)
\end{equation}
where we assume that $\vs{\theta}_0$ is such that $T(\vs{\theta}_{t})^\top T(\vs{\theta}_{t}) \not = 0$ for all $t$. Then, $\vs{\theta}_t \rightarrow \pm\vs{\nu}_1$ and $ \lambda_t \rightarrow \pm z_1$.

More precisely, let $$M=\max\left\{1, \frac{|z_1|^2}{|z_i|^2}, |z_1|\frac{\|\nu_i\|}{|z_i|} : i\in[k]\right\}\ \ \mbox{ and }\ \ \varepsilon_t = kM\left|\dfrac{z_2.\vs{\nu}_2^\top \vs{\theta}_0}{z_1.\vs{\nu}_1^\top \vs{\theta}_0}\right|^{2^t}. $$  
Then for all $t\geq 2$ such that $\varepsilon_t< \frac{1}{2},$ we have 
$$|e_t f_t \lambda_t-z_1|\leq 7|z_1|\varepsilon_t\ \ \mbox{ and }\ \ \|e_t f_t \vs{\theta}_t-\vs{\nu}_1\|\leq \varepsilon_t\left(||\vs{\nu}_1||+\sqrt{2}\right),$$
where $(e_t)_t$ and $(f_t)_t$ are two sequences defined in the proof and taking their values in $\{-1,1\}$.
\end{theorem}

\begin{proof}
Let us first define the square root of a complex number $z=re^{i\theta}$, where $-\pi<\theta\leq \pi$ and $r \geq 0$, by $z^{1/2} = r^{1/2}e^{i\frac{\theta}{2}}$, and note that $z / (z^2)^{1/2} = z^{-1} (z^2)^{1/2} = 1$ if $-\pi/2 < \theta \leq \pi/2$ and $-1$ otherwise.

Now, let $c_i = \vs{\nu}_i^\top \vs{\theta}_0$ for $i\in [k]$, $\widetilde{\vs{\theta}}_0 = \vs{\theta}_0$, $\widetilde{\vs{\theta}}_t = T(\widetilde{\vs{\theta}}_{t-1})$, and $\rho_t = (\widetilde{\vs{\theta}}_t^\top \widetilde{\vs{\theta}}_t)^\frac{1}{2}$ for all $t\geq 1$.
Check by induction on $t$ that, for all $t\geq 1$,
\begin{equation}
 \widetilde{\vs{\theta}}_t = \sum_{i=1}^k z_i^{2^t-1}c_i^{2^t} \vs{\nu}_i \label{eq:1}
\end{equation}

Let $e_t = \rho_{t+1} \rho_t^{-2} / \left(\rho_t^{-4} \rho_{t+1}^2\right)^{\frac{1}{2}}$, note that $e_t = \pm 1$, and check by induction that, for all $t\geq 2$,
\begin{equation}
\label{eq:1bis}
\vs{\theta}_t = e_t \frac{\widetilde{\vs{\theta}}_t}{\rho_t}.
\end{equation}

Let $\alpha_t =\rho_t^{-1} z_1^{2^t-1}c_1^{2^t}$. Using Eq.~\ref{eq:1} and Eq.~\ref{eq:1bis}, we obtain 
\begin{align*}
  e_t\lambda_t &=
  \rho_t^{-3}
  \sum_{i=1}^k z_i(z_i^{2^t-1}c_i^{2^t})^3
&=\alpha_t^3 \sum_{i=1}^k \frac{z_1^3}{z_i^2}\left(\frac{z_ic_i}{z_1c_1}\right)^{3\cdot 2^t}
&=\alpha_t^3z_1 \left[1+\sum_{i=2}^k \frac{z_1^2}{z_i^2}\left(\frac{z_ic_i}{z_1c_1}\right)^{3\cdot 2^t}\right],\mbox{ and}\\
  e_t\vs{\theta}_t &=\rho_t^{-1} \sum_{i=1}^k z_i^{2^t-1}c_i^{2^t} \vs{\nu}_i
&=\alpha_t\sum_{i=1}^k \frac{z_1}{z_i}\left(\frac{z_ic_i}{z_1c_1}\right)^{2^t}\vs{\nu}_i 
&=\alpha_t \left[\vs{\nu}_1+\sum_{i=2}^k \frac{z_1}{z_i}\left(\frac{z_ic_i}{z_1c_1}\right)^{2^t}\vs{\nu}_i\right].
\end{align*}
It can easily be checked that 
$$\left|\sum_{i=2}^k \frac{z_1^2}{z_i^2}\left(\frac{z_ic_i}{z_1c_1}\right)^{3\cdot 2^t}\right|\leq \varepsilon_t\ \ \mbox{ and }\ \ \left\|\sum_{i=2}^k \frac{z_1}{z_i}\left(\frac{z_ic_i}{z_1c_1}\right)^{2^t}\vs{\nu}_i\right\|\leq \varepsilon_t.$$
Moreover, it can be checked that
\begin{equation*}
  \alpha_t^{-1}=\frac{(\widetilde{\vs{\theta}}_t^\top\widetilde{\vs{\theta}}_t)^{1/2}}{z_1^{2^t-1}c_1^{2^t}}
=f_t\left(\frac{\widetilde{\vs{\theta}}_t^\top\widetilde{\vs{\theta}}_t}{(z_1^{2^t-1}c_1^{2^t})^2}
  \right)
  ^{1/2}
=f_t\left[1+\sum_{i=2}^k\frac{z_1^2}{z_i^2}\left(\frac{z_ic_i}{z_1c_1}\right)^{2^{t+1}}\right]^{1/2}
\end{equation*}
where $f_t = (z_1^{2^t-1}c_1^{2^t})^{-1} \left( (z_1^{2^t-1}c_1^{2^t})^2 \right) ^\frac{1}{2} = \pm 1$.
Using the hypothesis $\varepsilon_t < \frac{1}{2}$ and making use of Lemma~\ref{lemma_alpha} in Appendix~\ref{supmat:mainThm}, it follows that
$$|\alpha_t|\leq \sqrt{2},\ \ \  |f_t\alpha_t-1|\leq \varepsilon_t\ \ \mbox{ and }\ \ |f_t\alpha_t^3-1|\leq 4\varepsilon_t.$$ 
Finally, combining these inequalities, we obtain 
$$|e_t f_t\lambda_t-z_1|\leq 7|z_1|\varepsilon_t\ \ \mbox{ and }\ \ \|e_t f_t \vs{\theta}_t-\vs{\nu}_1\|\leq \varepsilon_t\left(||\vs{\nu}_1||+\sqrt{2}\right).$$



\end{proof}

This theorem directly yields an algorithm to recover the decomposition of a decomposable complex-valued tensor using the standard deflation technique.

It can be shown that if $\vs{\theta}_0$ is chosen at random in $\C^n$, the assumptions on $\vs{\theta}_t$ in the previous theorem are satisfied with probability one. We prove it here for the assumption $T(\vs{\theta}_t)^\top T(\vs{\theta}_t) \not = 0$,  similar arguments can be used for the assumption $|z_1.\vs{\nu}_1^\top \vs{\theta}_0| > |z_2.\vs{\nu}_2^\top \vs{\theta}_0| \geq \cdots \geq |z_k.\vs{\nu}_k^\top \vs{\theta}_0| > 0$.

\begin{lemma}
Using the definitions and under the hypothesis of Theorem~\ref{mainThm}, the set $S = \{ \vs{\theta}_0\in\C^n | \exists t \geq 0 : T(\vs{\theta}_{t})^\top T(\vs{\theta}_{t}) = 0\}$ has Lebesgue measure zero in $\C^n$.
\end{lemma} 
\begin{proof}
We use the notations of the previous proof. 

First note that $T(\vs{\theta}_{t-1})^\top T(\vs{\theta}_{t-1}) = \left(\widetilde{\vs{\theta}}_{t-1}^\top\widetilde{\vs{\theta}}_{t-1}\right)^{-2} \widetilde{\vs{\theta}}_t^\top\widetilde{\vs{\theta}}_t $ and $\widetilde{\vs{\theta}}_{t}^\top\widetilde{\vs{\theta}}_{t} = \sum_{i=1}^k (z_i^{2^t-1} (\vs{\nu}_i^\top \vs{\theta}_0)^{2^t})^2$. For a fixed $t$, the set 
$$S_t = \left\{\vs{\theta}\in\C^n : P_t(\vs{\theta}) = \sum_{i=1}^k (z_i^{2^t-1} (\vs{\nu}_i^\top \vs{\theta})^{2^t})^2 = 0\right\}$$ 
is the set of zeros of a multivariate polynomial. If $P_t$ is non-trivial (i.e. different from zero), it is a proper algebraic subvariety of $\C^n$ of dimension less than $n$, thus of Lebesgue measure 0. Since $S = \cup_{t = 0}^\infty S_t$, it is sufficient to show that $P_t$ is non-trivial for any index $t$.

Without loss of generality, we assume that there exists at least one $i\in [k]$ such that the first  component $\nu_{i,1}$ of the vector $\vs{\nu}_i$ is not null. Suppose that $P_t$ is null, then all of its monomials are null. In particular, the coefficient associated with $\theta_1^{2^{t+1}-1}\theta_j$, which is proportional to $\sum_{i=1}^k z_i^{2^{t+1}-2} \nu_{i,1}^{2^{t+1}-1} \nu_{i,j}$, is null for all $j\in [n]$. Let $\alpha_i = z_i^{2^{t+1}-2} \nu_{i,1}^{2^{t+1}-1}$ for $i\in [k]$, and note that since $z_i \not = 0$ for all $i\in [k]$, we cannot have all the $\alpha_i$ equal to zero. Thus, we have $\sum_i \alpha_i \nu_{i,j}=0$ for all $j\in [n]$, i.e. $\sum_{i=1}^k \alpha_i \vs{\nu}_i = \vs{0}$ which is in contradiction with the linear independence of $\{\vs{\nu}_i\}_{i=1}^k$.

\end{proof}

We can now state the following theorem, which summarizes the overall procedure to recover the parameters of a system of the form (\ref{system}) using pseudo-orthonormalization and the complex tensor power method. Note that this procedure generalizes the one proposed in \cite{HsuKakade}: if all the weights $w_1, \cdots, w_k$ are positive, the method we propose boils down to theirs.

\begin{theorem}\label{overallProc}
Let $\vs{\mu}_1,\cdots,\vs{\mu}_k \in \R^n$ be linearly independent, $w_1, \cdots, w_k \not= 0 \in \R$, $\v{M}_2 = \sum_{i=1}^k w_i\vs{\mu}_i\otimes\vs{\mu}_i$ and $\e{M}_3 = \sum_{i=1}^k w_i\vs{\mu}_i\otimes\vs{\mu}_i \otimes\vs{\mu}_i$. 
Let $\v{UDU}^\top$ be the eigendecomposition of $\v{M}_2$, $\v{W} = \v{UD}^{-\frac{1}{2}} \in \C^{n\times k}$ (see section~\ref{pseudo-orth}) and $(\v{W}^\top)^+ = \v{UD}^\frac{1}{2}$. 
Finally, let $\e{T} = \e{M}_3(\v{W},\v{W},\v{W})$ and let $\vs{\theta}_0$ be drawn at random in $\C^k$.

Then, using the definitions of $\vs{\theta}_t$ and $\lambda_t$ in Eq.~\ref{iterator}, we have 
$$\lim_t \frac{1}{\lambda_t^2} = w_j\ \ \mbox{ and }\ \ \lim_t \lambda_t (\v{W}^\top)^+ \vs{\theta}_t = \vs{\mu}_j$$
with probability one, where $j = \argmax_i \{\left| \vs{\mu}_i^\top \v{W} \vs{\theta}_0\right|\}$.
\end{theorem}

The indeterminacy on the sign of the coefficients in the pseudo-orthogonal decomposition $\e{T} = \sum_{i=1}^k w_i^{-\frac{1}{2}} \left( w_i^\frac{1}{2}\v{W}^\top\vs{\mu}_i\right)^{\otimes 3}$ vanishes when we recover the original parameters $w_i$ and $\vs{\mu}_i$.

\comment{
\begin{algorithm}[tb]
   \caption{Negative Mixture Estimation}
   \label{alg}
\begin{algorithmic}
   \STATE {\bfseries Input:} $\varepsilon\in\R$, $\widehat{\v{M}}_2 \in \bigotimes^2\R^n$ and $\e{\widehat{M}}_3 \in \bigotimes^3\R^n$
   \STATE{\bfseries Output:} $w_1,\cdots,w_k\in\R$, $\vs{\mu}_1,\cdots,\vs{\mu}_k\in \R^n$
   \STATE Compute the eigen decomposition $\v{UDU}^\top$ of $\widehat{\v{M}}_2 $;
   \STATE $\v{W} \leftarrow \v{UD}^{-\frac{1}{2}}$; $\e{T} \leftarrow \e{\widehat{M}}_3 (\v{W},\v{W},\v{W})$; $i\leftarrow 1$;
   \REPEAT
   \STATE Draw $\vs{\theta}$ at random in $\C^n$;
   \REPEAT
   \STATE $\vs{\theta} \leftarrow \e{T}(I,\vs{\theta},\vs{\theta})$; $\vs{\theta} \leftarrow \frac{\vs{\theta}}{(\vs{\theta}^\top\vs{\theta})^\frac{1}{2}}$;
   \UNTIL{convergence}
   \STATE $\lambda\leftarrow \e{T}(\vs{\theta},\vs{\theta},\vs{\theta})$; $\e{T} \leftarrow \e{T} - \lambda . \vs{\theta}^{\otimes 3}$;
   \STATE $ w_i \leftarrow 1 / \lambda^2$; $\vs{\mu}_i \leftarrow \lambda (\v{W}^\top)^+ \vs{\theta} $;
   \STATE $i\leftarrow i+1$
   \UNTIL $\|\e{T}\| \leq \varepsilon$
\end{algorithmic}
\end{algorithm}
}

\section{Learning Negative Mixtures of Spherical Gaussians}\label{sec:learning}

In this section, we extend the method described in Section~\ref{GaussMixSection} to estimate the parameters of a negative mixture of spherical Gaussians. 
Let $f(\v{x}) = \sum_{i=1}^k w_i \mathcal{N} (\v{x}; \vs{\mu}_i, \sigma^2_i \v{I})$ be the \textsc{pdf} of the random vector $\v{x}$, where $\vs{\mu}_i\in\R^n$ are the component means, $\sigma^2_i$ the component variances, and $w_i \not = 0$ the coefficients ($\sum_{i=1}^k w_i = 1$). Assuming that the component means are linearly independent, we have the following result which generalizes Theorem~\ref{GaussMixThm}.

\begin{theorem}
\label{NegGaussMixThm}
The average variance $\bar{\sigma}^2 = \sum_{i=1}^k w_i \sigma^2_i$ is an eigenvalue of the covariance matrix $\Esp [(\v{x} - \Esp [\v{x}]) (\v{x} - \Esp [\v{x}]) ^\top]$. Let $\v{v}$ be any unit-norm eigenvector corresponding to $\bar{\sigma}^2$. We have $\v{m}_1 = \sum_{i=1}^k w_i \sigma^2_i \vs{\mu}_i$, $\v{M}_2 = \sum_{i=1}^k w_i \vs{\mu}_i \otimes \vs{\mu}_i$, and $\e{M}_3 = \sum_{i=1}^k w_i \vs{\mu}_i \otimes \vs{\mu}_i \otimes \vs{\mu}_i$, where $\v{m}_1$, $\v{M}_2$ and $\e{M}_3$ are defined as in Theorem~\ref{GaussMixThm}.

Moreover, let $r$ be the number of negative eigenvalues of the matrix $\v{M} = \sum_{i=1}^k w_i (\vs{\mu}_i - \Esp [\v{x}]) \otimes (\vs{\mu}_i - \Esp [\v{x}]) $. Then $\bar{\sigma}^2$ is the $(r+1)$-th smallest eigenvalue of the covariance matrix.
\end{theorem}   

The proof of this theorem is given in Appendix~\ref{supmat:learning}, where we also show that $r = l$ or $l+1$, where $l$ is the number of negative coefficients $w_i$, i.e. $l = |\{w_i : i\in [k],\ w_i < 0\}|$. 

This theorem, combined with Theorem~\ref{overallProc}, yields a procedure to estimate the parameters of a negative mixture of spherical Gaussians: (i) compute the sample covariance matrix $\v{S}$, (ii) for each candidate eigenvalue of $\v{S}$ for $\bar{\sigma}^2$, estimate the tensors $\v{m}_1$, $\v{M}_2$ and $\e{M}_3$ on the data, (iii) compute estimations of the parameters using Algorithm~\ref{alg_negMix}, (iv) choose the model that maximizes the likelihood of the learning data.


\section{Experiments}
\label{sec:expe}
We illustrate the results presented above on the running example defined in Figure~\ref{fig_pdf_cvg} (left). The algorithm to estimate the parameters of a system of the form~(\ref{system}) from estimation of the tensors $\v{M}_2$ and $\e{M}_3$ is summarized in Figure~\ref{fig_algo} (left).

First, we run Algorithm~\ref{alg_negMix} on the exact tensors $\v{M}_2 = \sum_{i=1}^k w_i\vs{\mu}_i\otimes\vs{\mu}_i$ and $\e{M}_3 = \sum_{i=1}^k w_i\vs{\mu}_i\otimes\vs{\mu}_i \otimes\vs{\mu}_i$ with various initializations of $\vs{\theta}_0$ to extract the first eigenvector/eigenvalue pair. The corresponding parameters $w_i$ and $\vs{\mu}_i$ are always exactly recovered in less than 15 iterations, the average error over 500 initializations for those two parameters in function of the number of iterations is plotted in Figure~\ref{fig_pdf_cvg} (right).

Then, we test our algorithm in a learning setting. For various sizes (ranging from 1,000 to 400,000), we generate 100 datasets (using Algorithm~\ref{algo:sim}) and use the method described in the previous section to estimate the parameters of the negative mixture of Gaussians. The results are plotted in Figure~\ref{fig_algo} (right), where each point represents the average on the 100 datasets of the $l^2$-norm between the true parameters ($\v{U} = [\vs{\mu}_1\ \  \vs{\mu}_2]$ and $\v{w} = [w_1\ \ w_2]$) and the estimations.

For some of these datasets, our algorithm returns a decomposition involving complex valued vectors and weights; for the experiments, we only used the real parts in the error measure. The number of these pathological datasets decreases toward zero as their size increases.

\begin{figure}[]
\centering
\includegraphics[trim=8cm 3cm 2cm 0cm,scale=0.35]{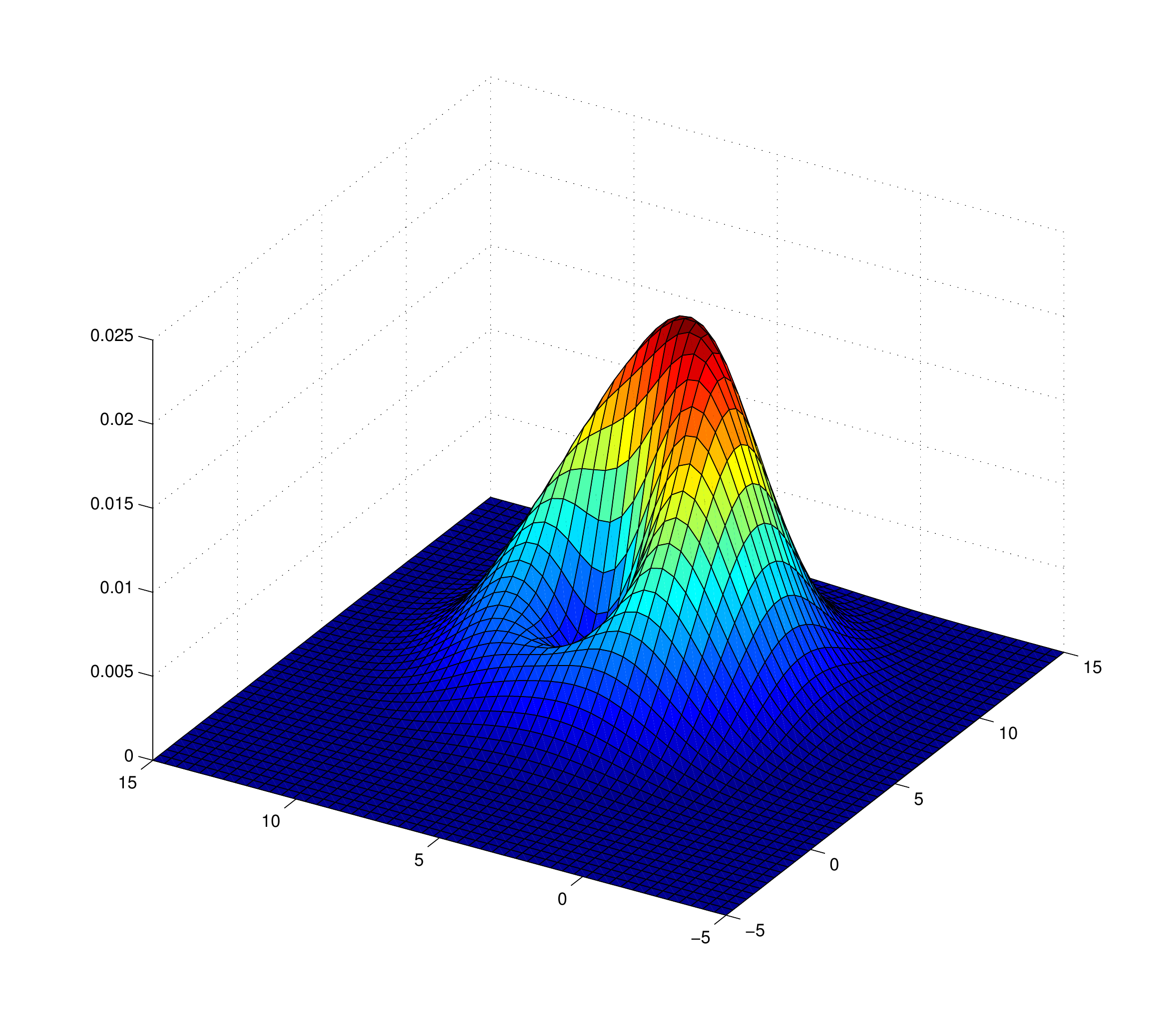}
\includegraphics[trim = 1cm 0cm 4cm 0cm, scale=0.4]{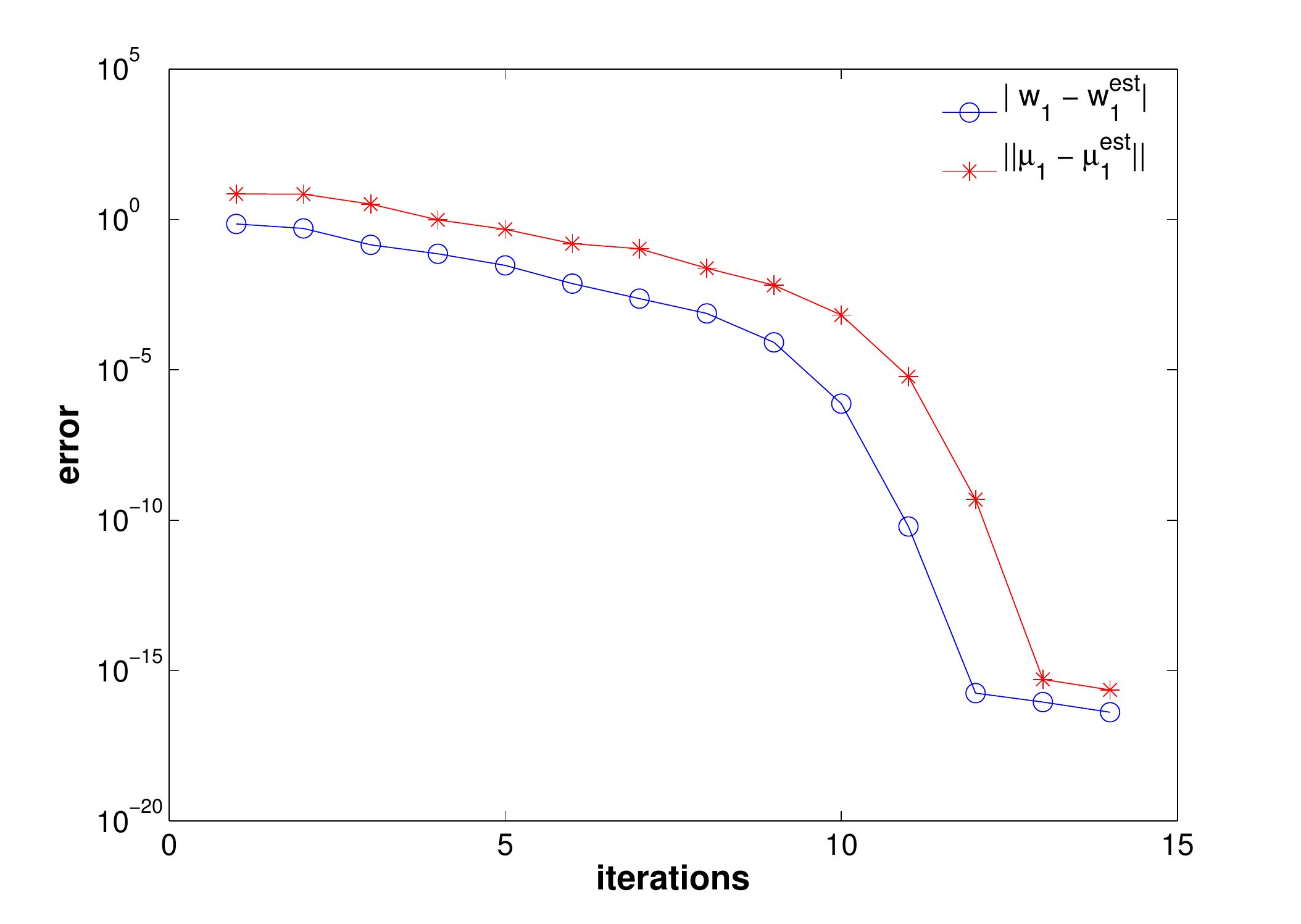}
\caption{(left) Density function of a negative mixture of spherical Gaussians with parameters $w_1 = 1.5$, $\vs{\mu}_1 = (11.4\ -3.4)^\top$, $\sigma^2_1 = 8$, $w_2 = -0.5$, $\vs{\mu}_2 =  (11.9\ - 1.9)^\top$ and $\sigma^2_2 = 4$. \newline(right) Convergence rate of the proposed method on the exact tensors $\v{M}_2$ and $\e{M}_3$.}
\label{fig_pdf_cvg}
\end{figure}

\begin{figure}[H]
\centering
\begin{minipage}{.5\textwidth}
  \centering
  \begin{algorithm}[H]
  \caption{Negative Mixture Estimation}
  \label{alg_negMix}
\begin{algorithmic}
   \STATE {\bfseries Input:} $k\in \N$, $\widehat{\v{M}}_2 \in \bigotimes^2\R^n$, $\e{\widehat{M}}_3 \in \bigotimes^3\R^n$
   \STATE{\bfseries Output:} $w_1,\cdots,w_k, \vs{\mu}_1,\cdots,\vs{\mu}_k$
   \STATE $\v{UDU}^\top \leftarrow \widehat{\v{M}}_2 $ ($k$-truncated eig. decomp.);
   \STATE $\v{W} \leftarrow \v{UD}^{-\frac{1}{2}}$; $\e{T} \leftarrow \e{\widehat{M}}_3 (\v{W},\v{W},\v{W})$; 
   \FOR{$i=1$ \TO $k$}
   \STATE Draw $\vs{\theta}$ at random in $\C^k$;
   \REPEAT
   \STATE $\vs{\theta} \leftarrow \e{T}(I,\vs{\theta},\vs{\theta})$; $\vs{\theta} \leftarrow \frac{\vs{\theta}}{(\vs{\theta}^\top\vs{\theta})^\frac{1}{2}}$;
   \UNTIL{stabilization}
   \STATE $\lambda\leftarrow \e{T}(\vs{\theta},\vs{\theta},\vs{\theta})$; $\e{T} \leftarrow \e{T} - \lambda . \vs{\theta}^{\otimes 3}$;
   \STATE $ w_i \leftarrow 1 / \lambda^2$; $\vs{\mu}_i \leftarrow \lambda (\v{W}^\top)^+ \vs{\theta} $;
   \ENDFOR
\end{algorithmic}
\end{algorithm}
\end{minipage}%
\begin{minipage}{.5\textwidth}
  \centering
\includegraphics[trim = 0cm 0.5cm 0cm 0cm, scale=0.45]{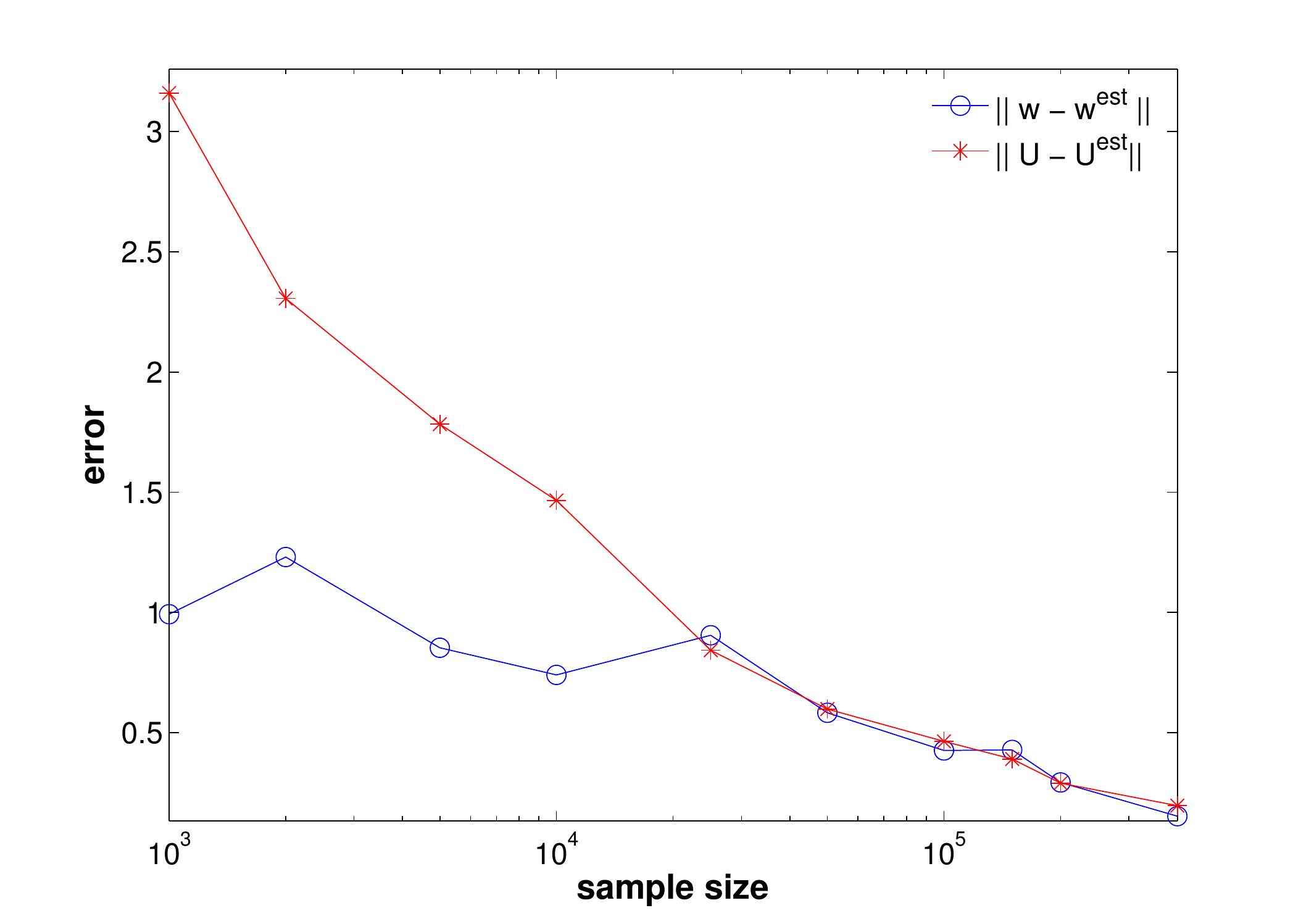}
  \label{fig:test2}
\end{minipage}
\caption{(left) Algorithm for the estimation of the parameters of a negative mixture model from estimation of the low-order moment tensors.\newline(right) Estimation error as a function of the dataset size.}\label{fig_algo}
\end{figure}


\section{Conclusion}

In this paper, we propose a first introductive study of negative mixture models. We argue that these models may appear naturally in several learning settings --- such as spectral learning of probability distributions on strings --- when the learning schemes rely on algebraic methods applied without positivity constraints (i.e. on fields, e.g. $\R$, rather than semi-fields, e.g. $\R_+$).

These models may seem difficult to handle, since allowing negative weights exclude the use of probabilistic methods such as EM. However, tensor decomposition techniques can be an appealing alternative. The complex tensor power method we propose, along with its application to the negative Gaussian mixture model, is a first step toward a deep understanding of these models and the elaboration of tools to use them. 

This work could be extended in several ways. First, other fields of machine learning where negative mixture models appear, or where their expressiveness can be useful, should be investigated. By extending the power method to complex valued tensors, we are able to propose an algorithm to estimate the parameters of such models, but the implications of using decomposition techniques on complex tensors need to be studied further. In particular, a deep robustness analysis of our method would help to understand its behavior in the learning setting. 

\comment{
In this paper, we showed that negative mixtures occur naturally in spectral learning and we proposed a method to tackle the estimation of such models. This work could be extended in several ways. First,  beyond the class of rational series, other fields of machine learning where negative mixture models occur, or where their expressiveness would be useful, could be found. Then, by extending the power method to complex valued tensors we were able to solve the problem at hand, but new ones appeared. For example, how then convergence of our method to a decomposition of a real tensor with complex valued vectors and weights should be handled ? Can we find the closest tensor decomposition over $\R$ to a decomposition over $\C$ ? Moreover, complex tensors appear as a byproduct of the whitening process, and the tensors in $\C^3$ obtained after whitening only lie in the subspace $(\R\cup i\R)^n$. This crucial information may be used to improve the robustness of our method. Finally, a deep robustness analysis of our method could help to understand its behavior in the learning setting. 
}
\newpage
\bibliography{icml2014}
\nocite{*}

\acks{This work has been carried out thanks to the support of the ARCHIMEDE Labex (ANR-11-LABX- 0033) and the A*MIDEX project (ANR-11-IDEX-0001-02) funded by the "Investissements d'Avenir" French government program managed by the ANR.}


\appendix

\section{Proofs and Complements}
\subsection{Rational probability distributions on strings}\label{supmat:rat}

Probabilistic automatas and HMMs define the same family of probability distributions on strings~\cite{DBLP:journals/pr/DupontDE05}. All these distributions are rational but the converse is false~\cite{MR0152021,DBLP:journals/fuin/DenisE08}. The simplest counter examples can be built on a one-letter alphabet and a dimension equal to 3.

Let $\Sigma=\{a\}$ be a one-letter alphabet. Let us define
a parametrized family of linear representations by 
$$\vs{\iota}=(\lambda,0,\sqrt{2}\lambda)^\top, M=\rho\left(
  \begin{array}{ccc}
    \cos{\alpha}&-\sin{\alpha}&0\\
\sin{\alpha}&\cos{\alpha}&0\\
0 &0 &1
  \end{array}
\right), \vs{\tau}=(1,1,1)^\top$$
where $\lambda>0$ and $0<\rho<1$. Let $r$ be the associated rational series. It can easily be seen
that $$r(a^n)=\rho^n\sqrt{2}
\lambda[\cos{(n\alpha-\pi/4)}+1)]\geq 0$$ for all $n$. We have
also $$r(\Sigma^*)=\lambda\left[\frac{1-\sqrt{2}\rho\cos{(\alpha-\pi/4)}}{1+\rho^2-2\rho\cos{\alpha}}+\frac{\sqrt{2}}{1-\rho}\right]$$
and $\lambda$ can always be chosen such that $r(\Sigma^*)=1$,
i.e. such that $r$ is a probability distribution. It can easily be seen that $r$ can be defined by a PA iff $\alpha/\pi\in \Q$.  

For example, if $\cos{\alpha}=3/5, \sin{\alpha}=4/5$, the corresponding distributions cannot be computed by a PA. 

If $\rho=0.5$, we have $\lambda=\frac{13}{6+26\sqrt{2}}$ and $\vs{\iota}\simeq(0.304,0,0.430)^\top$. The construction described in Section~\ref{sec:rat} yields the distributions $p^+$ and $p^-$ respectively defined by the following PAs: 
\small
$$\vs{\iota}^+=(0.4015,0,0.5985,0,0)^\top, \v{M}^+=\left(
  \begin{array}{ccccc}
0.300&0       &  0    &     0 &   0.173\\
    0.302   & 0.3   &      0      &   0   &      0\\
         0      &   0  &  0.5      &   0   &      0\\
         0    &0.7    &     0  &  0.3   &      0\\
         0     &    0   &      0   & 0.7 &   0.300\\
  \end{array}
\right), \vs{\tau}^+=(0.527,0.398,0.5,0,0)^\top,$$
$$\vs{\iota}^-=(0,0,1,0)^\top, \v{M}^-=\left(
  \begin{array}{cccc}
0.300      &   0   &      0 &   0.173\\
    0.302   & 0.3   &      0 &        0\\
         0   & 0.7  &  0.3   &      0\\
         0   &      0  &  0.7   & 0.300\\
  \end{array}
\right), \vs{\tau}^-=(0.527,0.398,0,0)^\top,$$
\normalsize
and the mixture parameters $s^+=1.4364$ and $s^-=-0.4364$. 

If $\rho=0.75$, the series $r^+$ and $r^-$ computed by Lemma~\ref{lem:difplus} do not converge. It is necessary to compute first a linear representation $(\vs{\iota}, \v{M}, \vs{\tau})$ of $r$ such that the series associated with  $(|\vs{\iota}|, |\v{M}|, |\vs{\tau}|)$ is convergent. This can be achieved using techniques described in~\cite{DBLP:journals/iandc/BaillyD11}. For example, we obtain the following linear representation 
\small
$$\vs{\iota}=(1,0,0,0,0,0)^\top, M=\left(
  \begin{array}{cccccc}
       0   & 0.5675   &      0   &      0      &   0  &       0\\
         0    &     0  &  0.7125    &     0     &    0   &      0\\
         0     &   0      &   0  &  0.9566    &     0      &   0\\
         0  &       0    &     0    &    0  &  0.9753   &      0\\
         0    &     0    &     0   &      0  &       0  &  0.8334\\
    0.5662 &  -0.1571  &      0       &  0      &   0  &  0.2750\\

  \end{array}
\right),$$
\normalsize 
and $\vs{\tau}=(0.4325
    0.2875,0.0434,0.0247, 0.1666, 0.3159)^\top$
from which the construction described in Section~\ref{sec:rat} can be applied. 

\subsection{Proof of Lemma~\ref{lemma_alpha}}
\label{supmat:mainThm}
\begin{lemma} 
\label{lemma_alpha}
Let $k>0$ and $z\in \C$ such that
$|z|<1/2$. Then, $$|(1+z)^{-k}-1|\leq 2|z|(2^k-1).$$
In particular, $$|(1+z)^{-1/2}-1|\leq |z|\textrm{ and }|(1+z)^{-3/2}-1|\leq 4|z|.$$
\end{lemma}
\begin{proof}
Let $f(z)=(1+z)^{-k}-1$ with $k>0$ and $|z|<1/2$;
$f'(z)=-k(1+z)^{-(k+1)}$. Let $\gamma:[0,1]\mapsto \C$
s.t. $\gamma(t)=tz$. 
We have
\begin{align*}
  (1+z)^{-k}-1&=\int_{\gamma}f'(y)dy=\int_0^1f'(\gamma(t))\gamma'(t)dt\\
&=-kz\int_0^1(1+tz)^{-(k+1)}dt.
\end{align*}
Therefore, 
\begin{align*}
  |(1+z)^{-k}-1|&\leq k|z|\int_0^1(1-t|z|)^{-(k+1)}dt\\
&=[(1-t|z|)^{-k}]_0^1=(1-|z|)^{-k}-1\\
&\leq 2|z|(2^k-1).
\end{align*}
Indeed, let $g(x)=(1-x)^{-k}-1-2x(2^k-1)$. It can be checked that
$g(0)=g(1/2)=0$ and that $g$ is convexe on $[0,1/2]$. 
    
\end{proof}

\subsection{Proof of Theorem~\ref{NegGaussMixThm}}\label{supmat:learning}

We will need the following results. The first one is a corollary of the \emph{Sylvester's Law of Inertia}. 

\begin{lemma}\label{lem:sylvester}
Let $\v{Q}\in\R^{n\times n}$ be a symmetric real matrix. Suppose that there exists a non singular matrix $\v{P}\in\R^{n\times n}$ and $w_1, \ldots w_n\in \R$ such that $\v{Q} = \v{P}^\top \v{D} \v{P}$ where $\v{D}=diag(w_1, \ldots, w_n)$, the diagonal matrix whose diagonal entries are $w_1, \ldots, w_n$. Then, the number of negative eigenvalues of $\v{Q}$ is equal to the number of negative coefficients $w_i$. 
\end{lemma} 

\begin{lemma}\emph{(Weyl's Inequality)} Let $\v{A}$ and $\v{B}$ be two $n\times n$ hermitian matrices. We have $\sigma_1(\v{A}) + \sigma_i(\v{B}) \leq \sigma_i(\v{A}+\v{B}) \leq \sigma_n(\v{A}) + \sigma_i(\v{B})$ for all $i\in [n]$, where $\sigma_i(\v{M})$ denotes the $i$-th smallest eigenvalue of $\v{M}$.
\end{lemma}

\begin{lemma}
\label{rank1sgn}
Let $\{\v{v}_i\}_{i=1}^k$ be a linearly dependent family of vectors of $\R^n$, where $n\geq k$, such that any of its subset of size $k-1$ is linearly independent. We consider the rank $k-1$ matrix $\v{M} = \sum_{i=1}^k w_i \v{v}_i\v{v}_i^\top$ where $w_1,\cdots, w_k \not = 0$. Let $l$ be the number of negative coefficients $w_i$. 
Then the first null eigenvalue of $\v{M}$ is either the $l$-th or the $(l+1)$-th smallest one.
\end{lemma} 

\begin{proof}
If $l=0$, then $\v{M}$ is positive semi-definite and $\sigma_1(\v{M})=0$. If $l=k$, then $\v{M}$ is negative semi-definite and $\sigma_k(\v{M})=0$.  

We suppose that $1\leq l\leq k-1$.

For $1\leq j\leq k$, let $\v{M}_j = \sum_{1\leq i\not = j \leq k}  w_i \v{v}_i\v{v}_i^\top$ and $ l_j$ be the number of negative coefficients in $\{w_i\}_{1\leq i \not = j \leq k}$. Let $V_j$ be the vector space spanned by $\{\v{v}_1, \ldots, \overline{\v{v}}_j, \ldots, \v{v}_k\}$, where the notation $\overline{\v{v}}_j $ means that $\v{v}_j$ is omitted.  Let $\vs{\nu}_k, \cdots, \vs{\nu}_n$ be a linearly independent family of vectors in $V_j^\perp$ and $\v{P}$ be the non singular $n\times n$ matrix $[\v{v}_1 \cdots, \overline{\v{v}}_j, \cdots, \v{v}_{k}, \vs{\nu}_k, \cdots, \vs{\nu}_n]^\top$. Clearly, $$\v{M}_j= \v{P}^\top diag(w_1,\cdots, \overline{w}_j, \cdots, w_{k}, 0,\cdots, 0) \v{P}$$
and therefore, from Lemma~\ref{lem:sylvester}, $l_j$ is the number of negative eigenvalues of $\v{M}_j$.

For any $j\in [k]$, we consider the decomposition $\v{M} = w_j \v{v}_j \v{v}_j^\top + \v{M}_j$, sum of two hermitian matrices. The first summand is a rank one matrix whose only non null eigenvalue has the same sign as $w_j$, and the second has $k-1$ non zero eigenvalues, among $l_j$ are negative.

Let $j$ be an index such that $w_j < 0$: from Weil's inequality $\sigma_i(\v{M}) \leq 0 + \sigma_i(\v{M}_j)$ for any $i\in[n]$. Since $\v{M}_j$ has $l_j=l-1$ negative eigenvalues, $\v{M}$ has at least $l-1$ negative eigenvalues. 

Let $j$ be an index such that $w_j > 0$: Weil's inequality gives $\sigma_i(\v{M}) \geq 0 + \sigma_i(\v{M}_j)$ for any $i\in[n]$, thus $\v{M}$ has at least $k-l-1$ positive eigenvalues, hence at most $l$ negative ones. 

Therefore, the first null eigenvalue of $\v{M}$ must be either the $l$-th or the $(l+1)$-th smallest one.
\end{proof}

Let $f(\v{x}) = \sum_{i=1}^k w_i \mathcal{N} (\v{x}; \vs{\mu}_i, \sigma^2_i \v{I})$ be the \textsc{pdf} of the random vector $\v{x}$, and let $l$ be the number of negative weights $w_i$. We can now prove Theorem~\ref{NegGaussMixThm}, along with the relation between $l$ and the position of the eigenvalue $\bar{\sigma}^2$ in the covariance matrix.

\begin{theorem*}
The average variance $\bar{\sigma}^2 = \sum_{i=1}^k w_i \sigma^2_i$ is an eigenvalue of the covariance matrix $\Esp [(\v{x} - \Esp [\v{x}]) (\v{x} - \Esp [\v{x}]) ^\top]$. Let $\v{v}$ be any unit-norm eigenvector corresponding to $\bar{\sigma}^2$. We have $\v{m}_1 = \sum_{i=1}^k w_i \sigma^2_i \vs{\mu}_i$, $\v{M}_2 = \sum_{i=1}^k w_i \vs{\mu}_i \otimes \vs{\mu}_i$, and $\e{M}_3 = \sum_{i=1}^k w_i \vs{\mu}_i \otimes \vs{\mu}_i \otimes \vs{\mu}_i$, where $\v{m}_1$, $\v{M}_2$ and $\e{M}_3$ are defined as in Theorem~\ref{GaussMixThm}.

Moreover, let $r$ be the number of negative eigenvalues of the matrix $\v{M} = \sum_{i=1}^k w_i (\vs{\mu}_i - \Esp [\v{x}]) \otimes (\vs{\mu}_i - \Esp [\v{x}]) $. Then $\bar{\sigma}^2$ is the $(r+1)$-th smallest eigenvalue of the covariance matrix.

Furthermore, $r$ is either equal to $l$ or $l+1$.
\end{theorem*}   

\begin{proof}
Most of the proof of this theorem for usual Gaussian mixtures in \cite{HsuKakade2} relies on the introduction of a discrete latent variable $h$: the sampling process is interpreted as first sampling $h$ with $\Prob[h=i] = w_i$, and then sampling $\v{x} = \vs{\mu}_h + \vs{z}_h$ where $\vs{z}_h$ is a multivariate Gaussian with mean $\vs{0}$ and covariance $\sigma^2_h \vs{I}$. Allowing negative weights in the mixture, we cannot use the same strategy, but it will be sufficient to note that $\Esp[g(\v{x})] = \sum_{i=1}^k w_i \Esp[g(\vs{\mu}_i + \v{z}_i)]$ for any function $g$, which is a direct consequence of the linearity of the expectation.

First, we need to identify the position of $\bar{\sigma}^2$ in the covariance matrix. Let $\bar{\vs{\vs{\mu}}} = \Esp[\v{x}] = \sum_{i=1}^k w_i \vs{\vs{\mu}_i}$. The covariance matrix of $\v{x}$ is 
\begin{equation*}
\Esp[(\v{x} - \bar{\vs{\mu}}) \otimes (\v{x} - \bar{\vs{\mu}})] = \sum_{i=1}^k w_i (\vs{\mu}_i - \bar{\vs{\mu}}) \otimes (\vs{\mu}_i - \bar{\vs{\mu}}) + \bar{\sigma}^2\v{I}.
\end{equation*}

Since the $\vs{\mu}_i$'s are linearly independent, $F = \{\vs{\mu}_i - \bar{\vs{\mu}}\}_{i=1}^k$ is a linearly dependent family of vectors of $\R^n$ such that any of its subset of size $k-1$ is linearly independent. It follows from Lemma~\ref{rank1sgn} that $0$ is either the $l$-th or $(l+1)$-th smallest eigenvalue of the matrix $\sum_{i=1}^k w_i (\vs{\mu}_i - \bar{\vs{\mu}}) \otimes (\vs{\mu}_i - \bar{\vs{\mu}})$, which implies that $\bar{\sigma}^2$ is the corresponding eigenvalue in the covariance matrix.

\comment{
 We prove that  In order to do so, we just need to show that, the result then directly follows from . 

First, since $\sum_{i=1}^k w_i (\vs{\mu}_i - \bar{\vs{\mu}}) = 0$, $F$ is linearly dependent. Now suppose that the subset $\{\vs{\mu}_i - \bar{\vs{\mu}}\}_{i=2}^k$ is linearly dependent, then there exists $\alpha_2, \cdots, \alpha_k$ not all zero such that $\sum_{i=2}^k \alpha_i (\vs{\mu}_i - \bar{\vs{\mu}}) = 0$, from which we can deduce (i) $\sum_{i=2}^k \alpha_i \vs{\mu}_i = \sum_{i=2}^k \alpha_i \sum_{j=1}^k w_j \vs{\mu}_j$.

Let $Z = \sum_{i=2}^k \alpha_i$ and note that $Z \not = 0$ (otherwise $\sum_{i=2}^k \alpha_i (\vs{\mu}_i - \bar{\vs{\mu}}) = \sum_{i=2}^k \alpha_i \vs{\mu}_i = 0$ which is in contradiction with the linear independence of $\{\vs{\vs{\mu}}_i\}_{i=1}^k$). It follows from (i) that $\vs{\mu}_1 = \frac{1}{w_1}( \sum_{i=2}^k \frac{\alpha_i}{Z}\vs{\mu}_i - \sum_{j=2}^k w_j \vs{\mu}_j) $ which is in contradiction with the linear independence of $\{\vs{\vs{\mu}}_i\}_{i=1}^k$.}

Note that the strict separation of $\bar{\sigma}^2$ from the other eigenvalues in the covariance matrix implies that every eigenvector corresponding to $\bar{\sigma}^2$ is in the null space of $\sum_{i=1}^k w_i (\vs{\mu}_i - \bar{\vs{\mu}}) \otimes (\vs{\mu}_i - \bar{\vs{\mu}})$, hence $\v{v}^\top  (\vs{\mu}_i - \bar{\vs{\mu}}) = 0$ for all $i\in [k]$.

We now express $\v{m}_1$, $\v{M}_2$ and $\e{M}_3$ in terms of the parameters $w_i$, $\sigma^2_i$ and $\vs{\mu}_i$. First,

\begin{align*}
\v{m}_1 &= \Esp[\v{x}(\v{v}^\top (\v{x} - \Esp [\v{x}]))^2]  \\
&= \sum_{i=1}^k w_i \Esp[(\vs{\mu}_i + \v{z}_i)(\v{v}^\top( \vs{\mu}_i - \bar{\vs{\mu}} + \v{z}_i))^2] \\
	&= \sum_{i=1}^k w_i \Esp[(\vs{\mu}_i + \v{z}_i)(\v{v}^\top\v{z}_i)^2]\ = \sum_{i=1}^k w_i\sigma^2_i \vs{\mu}_i.
\end{align*}

Next, since $\Esp[\v{z}_i\otimes \v{z}_i] = \sigma^2_i\v{I}$ for all $i\in [k]$, we have
\begin{align*}
\v{M}_2 &= \Esp[\v{x}\otimes\v{x}] - \bar{\sigma}^2 \v{I} \\
	&= \sum_{i=1}^k w_i \Esp[(\vs{\mu}_i + \v{z}_i) \otimes (\vs{\mu}_i + \v{z}_i) ] - \bar{\sigma}^2 \v{I}\\
	&= \sum_{i=1}^k w_i (\vs{\mu}_i \otimes \vs{\mu}_i + \Esp[\v{z}_i\otimes \v{z}_i]) - \bar{\sigma}^2 \v{I}= \sum_{i=1}^k w_i \vs{\mu}_i \otimes \vs{\mu}_i.
\end{align*}

Finally, writing $z_{ij}$ for the $j$-th component of the vector $\v{z}_i$, we have 

\begin{align*}
\sum_{i=1}^k w_i \Esp[\vs{\mu}_i \otimes \v{z}_i \otimes \v{z}_i] &= \sum_{i=1}^k w_i \sum_{p=1}^n\sum_{q=1}^n \Esp[z_{ip}z_{iq}] \vs{\mu}_i \otimes \v{e}_p \otimes \v{e}_q \\
	&= \sum_{i=1}^k w_i \sigma^2_i \sum_{j=1}^n \vs{\mu}_i \otimes \v{e}_j \otimes \v{e}_j \\
	&= \sum_{j=1}^n \v{m}_1 \otimes \v{e}_j \otimes \v{e}_j,
\end{align*}
where we used the fact that $\Esp[z_{ip}z_{iq}] = \delta_{pq}\sigma^2_i$ for all $i \in [k]$, $p,q\in [n]$. Using the same derivation, we have $\sum_{i=1}^k w_i  \Esp[\v{z}_i \otimes \vs{\mu}_i \otimes  \v{z}_i] = \sum_{j=1}^n  \v{e}_j \otimes \v{m}_1 \otimes \v{e}_j$ and $\sum_{i=1}^k w_i \Esp[ \v{z}_i \otimes  \v{z}_i \otimes \vs{\mu}_i ] = \sum_{j=1}^n  \v{e}_j \otimes \v{e}_j \otimes \v{m}_1 $. Hence,

\begin{align*}
\Esp[\v{x}^{\otimes 3}] &= \sum_{i=1}^k w_i \left(  
\vs{\mu}_i^{\otimes 3} + \Esp[\vs{\mu}_i\otimes \v{z}_i \otimes \v{z}_i] + 
\Esp[\v{z}_i \otimes \vs{\mu}_i\otimes \v{z}_i] + \Esp[\v{z}_i \otimes \v{z}_i \otimes \vs{\mu}_i] 
\right) \\
	&= \sum_{i=1}^k w_i \vs{\mu}_i^{\otimes 3} + \sum_{j=1}^n \left( 
\v{m}_1\otimes \v{e}_j \otimes \v{e}_j + 
\v{e}_j \otimes \v{m}_1\otimes \v{e}_j +  
\v{e}_j \otimes \v{e}_j \otimes \v{m}_1
\right)
\end{align*}
and $\e{M}_3 = \sum_{i=1}^k w_i \vs{\mu}_i \otimes  \vs{\mu}_i \otimes  \vs{\mu}_i$.

\end{proof}

\end{document}